\documentclass[11pt]{article}

\usepackage[preprint]{acl}

\usepackage{times}
\usepackage{latexsym}
\usepackage{amsmath}
\usepackage{amssymb}
\usepackage{amsthm}
\usepackage{algorithm}
\usepackage{algorithmic}
\newtheorem{lemma}{Lemma}
\newtheorem{proposition}{Proposition}
\usepackage{array}
\usepackage{multirow}
\usepackage{booktabs}
\usepackage{enumitem}

\usepackage[T1]{fontenc}

\usepackage[utf8]{inputenc}

\usepackage{microtype}

\usepackage{inconsolata}

\usepackage{graphicx}

\title{GraphACM: Graph-based Action Consistency Measurement for Uncertainty Quantification in LLM Agents}

\author{
  \textbf{Yaxin Ma\textsuperscript{1}},
  \textbf{Wenbo Zhang\textsuperscript{2}},
  \textbf{Satya Krishna Pothapragada\textsuperscript{1}},
  \textbf{Alina Zare\textsuperscript{1}}
\\
\\
  \textsuperscript{1}University of Florida, Gainesville, FL, USA \\
  \textsuperscript{2}University of California, Irvine, CA, USA
\\
  \small{
    \textbf{Correspondence:} \href{yaxinma@ufl.edu}{yaxinma@ufl.edu}
  }
}

\begin{document}
\maketitle
\begin{abstract}
Behavioral consistency across repeated executions provides an informative signal for assessing the reliability of large language model (LLM) agents. However, existing approaches typically summarize consistency with a single task-level score, without identifying where behavioral differences occur or whether actions are compared under equivalent decision contexts. We propose a graph-based framework for fine-grained behavioral consistency measurement and uncertainty quantification. Given multiple rollouts of the same task, we merge equivalent states into a directed, action-labeled graph, enabling action comparisons at shared decision states. We introduce node-level and global measures of action consistency, using the global consistency score as a proxy for uncertainty to predict task success or failure. Experiments across three LLMs on ALFWorld, AppWorld, and WebShop demonstrate that our measure is positively associated with task success and achieves an average AUROC of 0.811 for task-success prediction, outperforming the strongest behavioral-consistency baseline (0.729) and verbalized confidence (0.769). These results highlight the effectiveness of state-conditioned behavioral consistency and pinpoint where agent actions agree. 
\end{abstract}

\section{Introduction}
\label{sec:introduction}
Large language model agents extend text generation to multi-step interaction with external environments~\citep{wang2024survey}. Frameworks such as ReAct interleave reasoning, actions, and environmental observations to execute tasks~\citep{yao2023react}. These capabilities support applications ranging from software development~\citep{yang2024sweagent} to personal-assistant tasks~\citep{trivedi2024appworld} and scientific research~\citep{lu2024aiscientist}.

As agents move toward real-world deployment, how to evaluate them and when to trust them remain open challenges~\citep{kapoor2025agents, kirchhof2025position}. Evaluating task success based on the final outcome is a standard approach, but it reveals little about the intermediate action behavior or whether the agent would succeed again on the same task. Recent work therefore studies \emph{behavioral consistency}, which measures the degree to which an agent behaves similarly across repeated executions of the same task. Behavioral consistency has been used as a signal of reliability~\citep{rabanser2026reliability}, with higher consistency indicating greater reliability. It can also be used for uncertainty quantification~\citep{mehta2026agents}: the more the executions agree, the more confidence the agent's behavior warrants.

Existing approaches typically compare behavioral consistency at the outcome or sequence level. Outcome-level consistency assess agreement among final outputs without characterizing intermediate decisions~\citep{yao2025taubench}. Sequence-level consistency quantify agreement across complete action sequences, such as counting the number of unique sequences~\citep{mehta2026agents} or computing Discrete Semantic Entropy (DSE) over clusters of identical sequences~\citep{ye2026function}. Such exact-match measures treat sequences differing by even a single inserted action as distinct, making them sensitive to minor shifts in action sequences. Other methods allow partial agreement: Tool Sequence Similarity (TSS) uses edit-based similarity between tool-name sequences, while Argument Consistency (AC) compares arguments of step-aligned tool calls~\citep{yagubyan2026consistent}. Action-Set Consistency (ASC) and Action-Distribution Consistency (ADC) discard temporal order by comparing action-type sets and frequency distributions, respectively~\citep{bouchard2026beyond}.

\begin{figure*}
    \centering
    \includegraphics[width=0.95\linewidth]{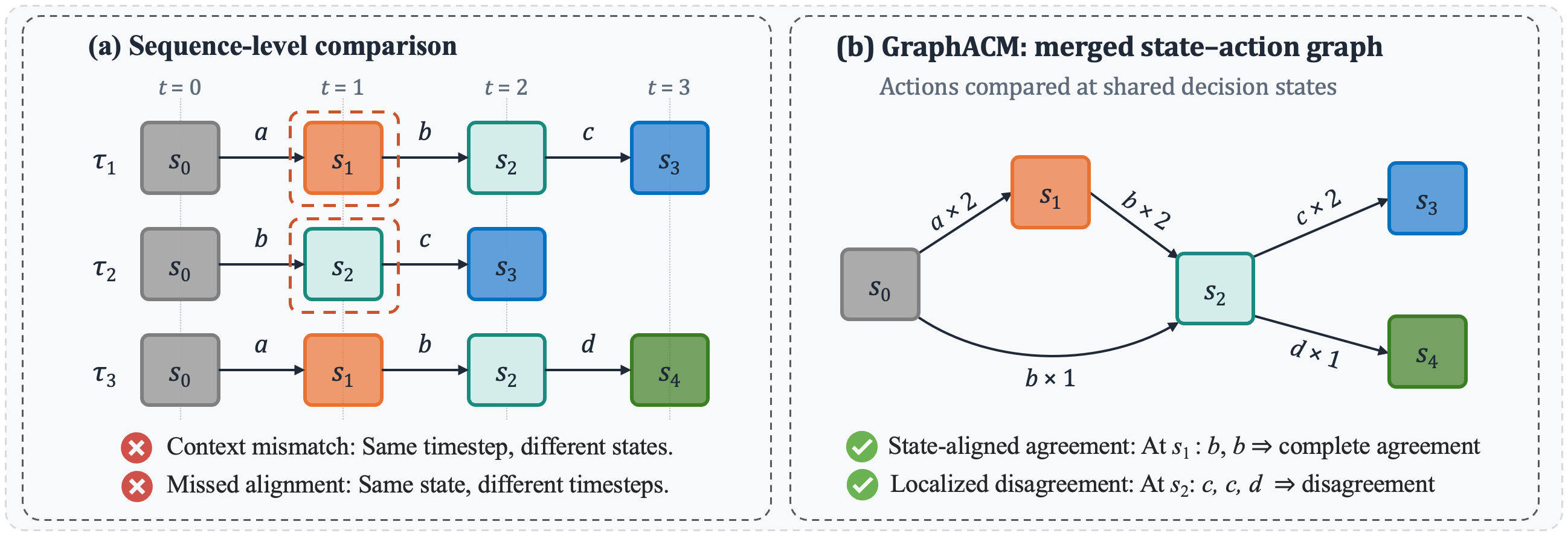}
        \caption{
Illustration of sequence-based and graph-based
behavioral consistency measurement.
(a) Position-wise comparison may align actions
originating from different decision states, while
equivalent states can occur at different timesteps.
(b) GraphACM merges equivalent states across rollouts
and compares actions at shared decision states,
identifying consistent decisions
and localized behavioral disagreement.
}
    \label{fig:illustration}
\end{figure*}

Despite these differences, existing approaches leave two important gaps. First, they primarily summarize behavioral consistency with a single task-level score, without identifying the intermediate decisions at which executions agree or diverge. Second, they establish action correspondence through sequence position or aggregate statistics, rather than the decision contexts in which actions were selected. Actions at the same step may respond to different situations, while equivalent situations may arise at different steps. Although repeated executions begin from the same initial environment, once trajectories diverge, their subsequent actions are no longer taken under the same conditions. Consequently, comparing actions without accounting for their source states can conflate differences in encountered situations with differences in the agent's decision-making behavior.

Recent graph-based reinforcement learning (RL) methods have demonstrated the value of aggregating agent trajectories through shared states. GraphGPO~\citep{cheng2026graphgpo} and G2PO~\citep{wang2026g2po} construct state-transition graphs by merging equivalent states across rollouts, with actions represented as transitions. These graphs support fine-grained credit assignment and advantage estimation for policy optimization. Although these methods also leverage multiple sampled trajectories, their objective is to optimize agent policies through training, whereas ours is to measure the behavioral consistency of a fixed agent at inference time without requiring outcome labels.

Motivated by these graph-based representations, we propose \emph{GraphACM} (Graph-based Action Consistency Measurement), a framework for fine-grained behavioral consistency measurement and uncertainty quantification in LLM agents. Given multiple rollouts of the same task, we construct a directed, action-labeled graph by merging equivalent states into shared nodes and representing actions as transitions, enabling action comparisons under equivalent decision contexts. At each shared node, we measure action agreement across
rollouts and aggregate these local measurements into a
task-level consistency score. As illustrated in
Figure~\ref{fig:illustration}, GraphACM addresses both
limitations by aligning actions based on their source
states and localizing behavioral agreement and divergence at intermediate decision points.

Our contributions are threefold:
\begin{enumerate}[leftmargin=*, itemsep=2pt, topsep=3pt, parsep=0pt, partopsep=0pt]

\item We propose GraphACM, a graph-based framework for behavioral consistency analysis that aggregates trajectories through shared states, enabling context-conditioned action comparisons.
    
    \item We formulate node-level action-consistency measures that quantify agreement at shared states and aggregate these local measurements into a task-level score, supporting fine-grained behavioral analysis and uncertainty estimation.
    
    \item We evaluate GraphACM across three LLMs and three agent benchmarks; our method outperforms existing baselines in most settings, demonstrating its superior performance in task-success discrimination.

\end{enumerate}

\section{GraphACM: Graph-based Action Consistency Measurement}

\subsection{Problem Formulation}
\label{sec:problem}

We consider an LLM agent with policy $\pi$ interacting with an environment
$\mathcal E$ to complete a task $g$ under a ReAct-style protocol
\citep{yao2023react}. For a fixed task with same initial condition, we independently sample $k\geq2$ rollouts. Rollout $i$ is
\begin{equation*}
\begin{aligned}
\tau^{(i)}
&=
\bigl(
o_0^{(i)},a_0^{(i)},o_1^{(i)},\ldots,
a_{T_i-1}^{(i)},o_{T_i}^{(i)}
\bigr),\\
&\qquad i \in \{1,\ldots,k\},
\end{aligned}
\label{eq:trajectory}
\end{equation*}

where $o_t^{(i)}$ is the observation available before action
$a_t^{(i)}$, and $o_{t+1}^{(i)}$ is the resulting observation.
$T_i$ is the number of recorded actions before termination or the
step limit; intermediate reasoning is not included. Our goal is to
measure behavioral consistency across these rollouts without using
their success labels.

\subsection{Proposed Method: GraphACM}
Our method provides a fine-grained, \emph{state-aligned} measure of behavioral consistency by comparing actions taken from equivalent states across repeated rollouts. A graph naturally supports this comparison by merging equivalent states into shared nodes and representing actions as labeled edges, revealing where trajectories share states and where their actions diverge. We first describe the graph construction, then introduce node-level and task-level action consistency measures, followed by uncertainty quantification and theoretical analysis.

\subsubsection{Graph Construction}
\label{sec:graph}
We merge the $k$ rollouts of each task into a directed,
action-labeled graph, where nodes represent abstract states
and edges represent actions, enabling actions to be aligned
by their source states.

\paragraph{Abstract states.}
The observation $o_t^{(i)}$ alone may not fully characterize
the agent's decision context, as the same environment can be
encountered with different information acquired from previous
interactions. We therefore define an abstract state
\begin{equation*}
z_t^{(i)}=\phi(e_t^{(i)},h_t^{(i)}),
\label{eq:state}
\end{equation*}
where $e_t^{(i)}$ represents the execution-relevant
environment context and $h_t^{(i)}$ summarizes information from history not captured by the
current environment context. States are considered
equivalent when their canonical representations match
under $\phi$. Benchmark-specific definitions and
canonicalization procedures are provided in
Appendix~\ref{app:graph}.

\paragraph{Node and edge.}
We represent each state as a node and each action as a directed transition to the subsequent state. Identical states across $k$ different rollouts are merged into shared nodes. The
resulting vertex set is $\mathcal V=\bigcup_{i=1}^{k}
\{z_t^{(i)}\mid 0\leq t\leq T_i\}$. For each node $v\in\mathcal V$, we define
$I_v=\{i\in\{1,\ldots,k\}\mid
\exists\,t<T_i,\;z_t^{(i)}=v\}$ as the rollouts that visit state $v$ and $n_v=|I_v|$ as their number. For each $i\in I_v$, $a_i(v)$ denotes the action
taken from $v$. Each transition is an action-labeled edge,
with edge set
$\mathcal E_G=
\{(z_t^{(i)},a_t^{(i)},z_{t+1}^{(i)})
\mid 1\leq i\leq k,\ 0\leq t<T_i\}$.

\paragraph{Merged rollout graph.}
The resulting graph
$\mathcal G=(\mathcal V,\mathcal E_G)$
is directed and action-labeled,
with each rollout corresponding to a walk through it.
By merging equivalent states, the graph reveals where
rollouts converge or diverge and enables comparison
of actions taken under equivalent abstract decision contexts.

\subsubsection{Consistency Measurement}
\label{sec:measure}

Given the constructed graph, actions originating from the same node are comparable under the chosen state abstraction. We measure action consistency at each node and aggregate these local measures to obtain task-level consistency.

\paragraph{Node-level action consistency.}
For a node $v$ visited by $n_v \geq 2$ rollouts, we define the node-level consistency $S_v$ as the average pairwise action agreement:
\begin{equation}
    S_v = \frac{1}{\binom{n_v}{2}}
    \sum_{\substack{i<j\\i,j\in I_v}}
    \mathrm{sim}\big(a_i(v),a_j(v)\big),
    \label{eq:Sv}
\end{equation}
where $I_v$ is the set of rollouts that take an action from node $v$, $n_v=|I_v|$, and $a_i(v)$ denotes the representative action taken by rollout $i$ at $v$. We use exact action matching:
\begin{equation*}
\mathrm{sim}\big(a_i(v),a_j(v)\big)=\mathbf{1}\big[a_i(v)=a_j(v)\big],
\end{equation*}
where equality requires both the action name and arguments to match. Thus, $S_v\in[0,1]$ with higher values indicating greater action agreement. For nodes visited by only one rollout, we set $S_v=0$, as no cross-rollout comparison is available.

\paragraph{Task-level action consistency.}
We aggregate node-level consistency into a global score:
\begin{equation}
  S_G=\frac{1}{k\,\bar T}\sum_v w_v\,S_v,
  \label{equ: global_consistency}
\end{equation}
where $w_v=n_v$ is the number of rollouts visiting node $v$, and $\bar T=\frac{1}{k}\sum_{i=1}^{k}T_i$ is the average rollout length, with $T_i$ denoting the number of actions in rollout $i$. Weighting each node by the number of visiting rollouts gives greater importance to states shared across more executions, while normalization by the total number of actions accounts for differences in trajectory length. The resulting score satisfies $S_G\in[0,1]$,
as proved in Appendix~\ref{app:sg_bound}.

\subsubsection{Uncertainty Quantification}
Motivated by prior work using behavioral consistency as a
signal of uncertainty for failure prediction
\citep{mehta2026agents},
we define task-level uncertainty as $U = 1-S_G$. Here, $U\in[0,1]$. Higher action consistency reflects
more stable decision-making across repeated executions,
suggesting lower uncertainty. Conversely, lower consistency
indicates greater behavioral variability and thus higher
uncertainty. The overall algorithm is summarized in Algorithm~\ref{alg:GraphACM}.

\begin{algorithm}[!h]
\caption{GraphACM}
\label{alg:GraphACM}
\begin{algorithmic}[1]

\STATE \textbf{Input:} Task $g$, policy $\pi$, environment
$\mathcal E$, number of rollouts $k$, state abstraction $\phi$.
\STATE \textbf{Output:} Task-level uncertainty $U$.

\STATE Sample $k$ independent rollouts
$\{\tau^{(i)}\}_{i=1}^{k}$ using $\pi$ on task $g$.

\STATE Construct abstract states
$z_t^{(i)}=\phi(e_t^{(i)},h_t^{(i)})$ for all rollouts.

\STATE Construct $\mathcal G=(\mathcal V,\mathcal E_G)$
by merging equivalent states and identical action-labeled transitions.

\FOR{$v\in\mathcal V$}
    \STATE Compute rollout support
    $I_v$ and $n_v=|I_v|$.

    \STATE For each $i\in I_v$, let $a_i(v)$ be
    the action from $v$ in rollout $i$.

    \STATE Compute $S_v$ using Eq.~\ref{eq:Sv}
    if $n_v\geq2$; otherwise set $S_v=0$.
\ENDFOR

\STATE Compute global consistency $S_G$
using Eq.~\ref{equ: global_consistency}.

\STATE \textbf{return} $U=1-S_G$.

\end{algorithmic}
\end{algorithm}

\subsection{Theoretical Analysis}

\begin{lemma}[Unbiasedness of local consistency]
\label{lem:local-consistency}
For a shared state $v$ with $n_v\geq2$, let
$p_v(a)$ denote the probability that a rollout
visiting $v$ selects representative action $a$.
Assume the representative actions across rollouts
are independently and identically distributed
according to $p_v$.
Then $S_v$ is an unbiased estimator of the
probability that two independent actions
$A,A'\sim p_v$ agree:
\begin{equation}
    \mathbb{E}[S_v]
    = \Pr(A=A')
    = \sum_a p_v(a)^2,
    \label{eq:local-consistency-expectation}
\end{equation}
\end{lemma}
This result connects local consistency to Rényi entropy
of order two~\citep{renyi1961measures}, $H_2(p_v)=-\log\sum_a p_v(a)^2$
since $\mathbb{E}[S_v]=\exp(-H_2(p_v))$. Thus, higher expected local consistency corresponds to
lower entropy in the conditional action distribution.
The proof is provided in
Appendix~\ref{app:local-consistency-proof}.

\begin{figure*}
    \centering
    \includegraphics[width=1\linewidth]{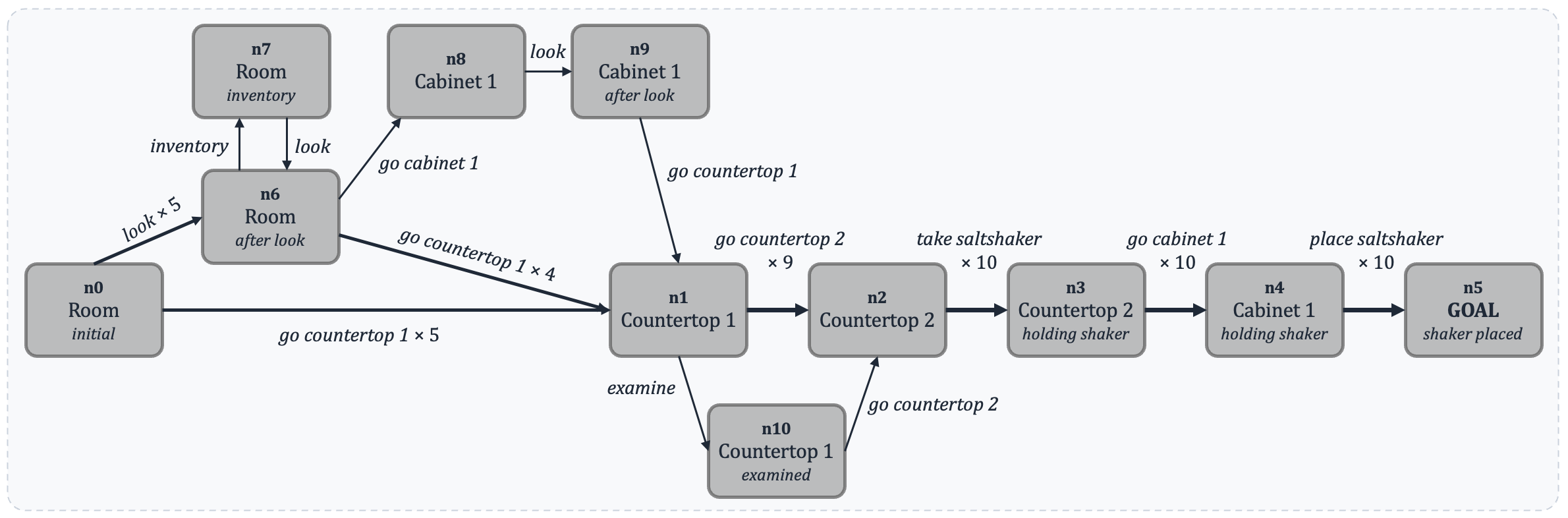}
    \caption{Merged state--action graph for an ALFWorld task
(``put some saltshaker on cabinet'') using Qwen3.8-27B
with $k=10$ successful rollouts. Nodes represent states,
and directed edges represent actions, with $\times$
indicating the number of rollouts taking each action.}
  \label{fig:graph_example}
\end{figure*}

\section{Experiments}

\subsection{Experiment Settings}

\paragraph{Benchmarks and LLM Models.}
We evaluate our method on three agent benchmarks:
ALFWorld~\citep{shridhar2021alfworld}, a text-based
household environment;
AppWorld~\citep{trivedi2024appworld}, a Python-based
API interaction environment; and
WebShop~\citep{yao2022webshop}, a text-based shopping
environment.
We use Qwen3.8-27B~\citep{qwen2026qwen38},
Gemma-4-31B-IT~\citep{google2026gemma4}, and
GPT-5.4-mini~\citep{openai2026gpt54mini} with the
verl-agent harness~\citep{feng2026gigpo}.
For each task--model pair, we independently sample
$k=10$ rollouts from the same initial state at
temperature $1.0$, without using hidden reasoning.
Episodes are limited to 50 steps for ALFWorld and
AppWorld and 30 for WebShop.
Dataset subsets and implementation details appear in
Appendix~\ref{app:exp-details}.
Table~\ref{tab:bas_res} in Appendix~\ref{app:basic}
reports success rates and mean trajectory lengths.

\paragraph{Labels and metrics.}
For each task--policy pair, we compute a task-level
consistency score from $k=10$ rollouts without using
outcome labels. A task is labeled \emph{successful}
if at least half of its rollouts succeed, and
\emph{failed} otherwise.
We report AUROC with success as the positive class,
and Spearman's rank correlation $\rho$ between
the consistency score and the empirical task success rate.

\paragraph{Baselines.}
We compare against six black-box behavioral-consistency
baselines that require no access to model internals:
{Unique Action Sequences (UAS)}
~\citep{mehta2026agents},
{Discrete Semantic Entropy (DSE)}
~\citep{ye2026function},
{Tool Sequence Similarity (TSS)} and
{Argument Consistency (AC)}
~\citep{yagubyan2026consistent},
and {Action-Set Consistency (ASC)} and
{Action-Distribution Consistency (ADC)}
~\citep{bouchard2026beyond}.
We additionally include {Verbalized Confidence (VC)}
as a non-consistency baseline, where the agent assesses
the success of each rollout through a binary Yes/No
judgment and a confidence probability~\citep{tian2023just}.

\begin{figure}[t]
    \centering
    \includegraphics[width=\linewidth]{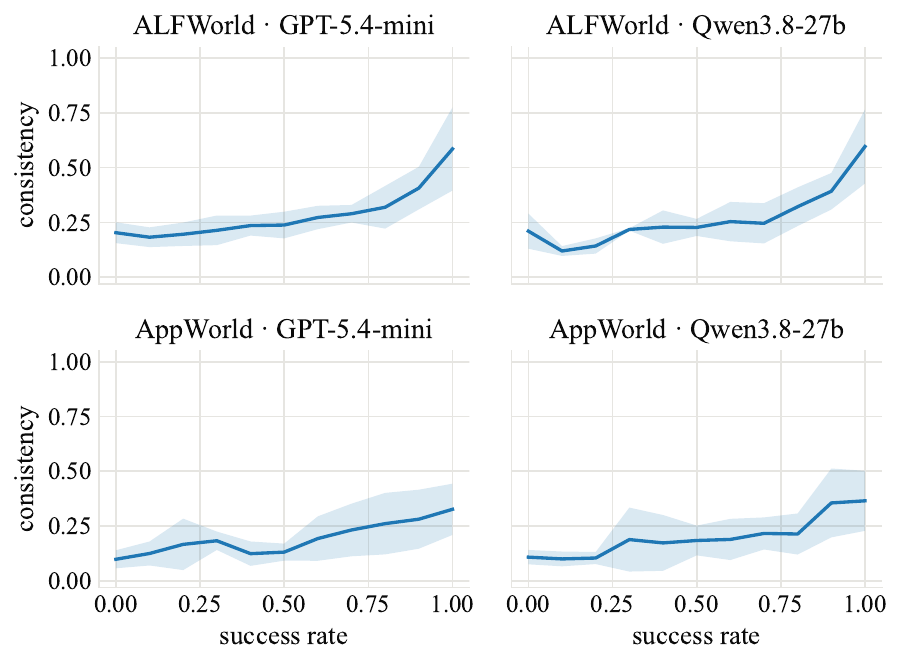}
    \caption{
Global action consistency versus empirical task success
rate. Lines show mean consistency across tasks at each
success-rate level; shaded bands indicate $\pm 1$
standard deviation.
}
\label{fig:consistency_vs_success}
    \label{fig:consistency_vs_success_rate}
\end{figure}

\subsection{Constructed graph and Local Analysis}
Figure~\ref{fig:graph_example} illustrates the merged
graph for an ALFWorld task, where 60 actions across
ten successful rollouts are merged into 11 unique
states and 14 distinct edges.
At the initial state, five rollouts choose
\texttt{look}, while the remaining five navigate
directly to \texttt{countertop 1}. When a rollout revisits a node, we use its last action
from that node as its representative action. Node-level consistency reveals that behavioral
disagreement is concentrated in the exploration phase:
$S_{n_0}=0.44$ at the initial state,
$S_{n_6}=0.60$ after \texttt{look}, and
$S_{n_1}=0.80$ at \texttt{countertop 1}. All ten rollouts then reconverge at \texttt{countertop 2} and follow identical actions thereafter, yielding $S_v=1$ at each shared decision state. This example demonstrates how state alignment
captures shared decision contexts across different
trajectory lengths and localizes behavioral
agreement and divergence at intermediate states.

\begin{table*}[t]
\centering

\begingroup
\setlength{\tabcolsep}{5pt}
\renewcommand{\arraystretch}{0.95}

\begin{tabular}{@{}lcccccccccc@{}}
\hline
{Methods}
& \multicolumn{3}{c}{{ALFWorld}}
& \multicolumn{3}{c}{{AppWorld}}
& \multicolumn{3}{c}{{WebShop}}
& {Average} \\

\cline{2-4}\cline{5-7}\cline{8-10}

&
{GPT} & {Qwen} & {Gemma}
&
{GPT} & {Qwen} & {Gemma}
&
{GPT} & {Qwen} & {Gemma}
& \\

\hline

UAS
& 0.655 & 0.741 & 0.866
& 0.476 & 0.504 & 0.531
& 0.702 & 0.592 & 0.618
& 0.632 \\

DSE
& 0.655 & 0.741 & 0.866
& 0.476 & 0.504 & 0.531
& \underline{0.703} & 0.596 & 0.613
& 0.632 \\

TSS
& 0.523 & 0.777 & 0.670
& 0.847 & 0.849 & \textbf{0.754}
& 0.519 & 0.553 & 0.571
& 0.674 \\

AC
& 0.730 & 0.728 & 0.777
& 0.819 & \textbf{0.888} & 0.704
& 0.683 & 0.611 & 0.618
& 0.729 \\

ASC
& 0.509 & \underline{0.900} & 0.773
& \textbf{0.878} & 0.835 & 0.722
& 0.500 & 0.500 & 0.500
& 0.680 \\

ADC
& 0.439 & 0.677 & 0.679
& 0.841 & 0.746 & 0.687
& 0.630 & 0.594 & 0.592
& 0.654 \\

VC
& \underline{0.829} & 0.853 & \underline{0.890}
& 0.791 & 0.873 & \underline{0.727}
& 0.633 & \textbf{0.667} & \underline{0.660}
& \underline{0.769} \\


{Ours}
& \textbf{0.907} & \textbf{0.962} & \textbf{0.922}
& \underline{0.857} & \underline{0.883} & \textbf{0.754}
& \textbf{0.718} & \underline{0.636} & \textbf{0.664}
& \textbf{0.811} \\

\hline
\end{tabular}

\endgroup
\caption{AUROC comparison of our method and baselines
for task-success prediction across ALFWorld, AppWorld,
and WebShop using GPT-5.4-mini, Qwen3.8-27B, and
Gemma-4-31B. Best results are shown in \textbf{bold} and second-best results are \underline{underlined}.}

\label{tab:main_results}
\end{table*}

\begin{figure*}
    \centering
    \includegraphics[width=1\linewidth]{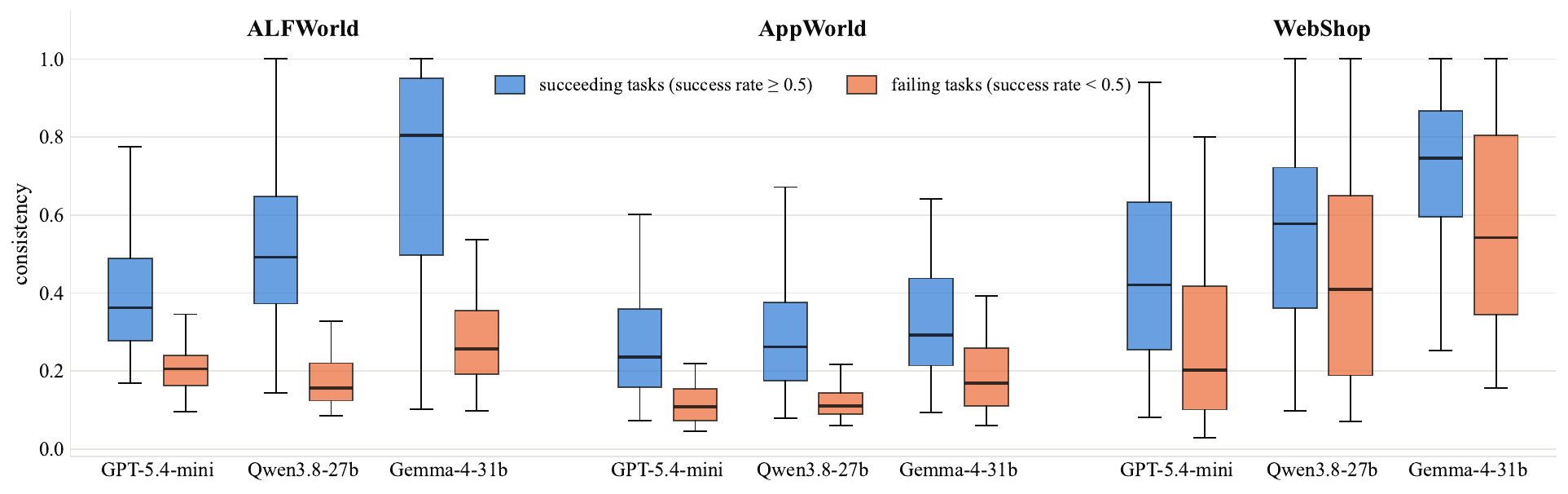}
    \caption{Consistency scores of successful and failed tasks across benchmarks and LLMs. Boxes show the median and interquartile range; whiskers extend to the furthest task within $1.5$ IQR of the box.}
    \label{fig:consistency_succ_fail}
\end{figure*}

\subsection{Consistency and task success rate}
\label{sec:consistency_task_succ_rate}

We examine the relationship between global action
consistency and task success rate. For each task--model pair, we sample $k=10$ trajectories and define the \emph{success rate} as the fraction that succeed. Figure~\ref{fig:consistency_vs_success_rate} presents results
for four model--benchmark pairs; all nine are reported
in Figure~\ref{app_fig:consistency_success_rate} in Appendix. On ALFWorld and AppWorld, consistency is positively
associated with success rate across all six pairs
(Spearman's $\rho=0.80$--$0.83$ on ALFWorld and
$0.57$--$0.74$ on AppWorld).
The increase is most pronounced at high success rates,
suggesting that consistency is particularly informative
for distinguishing reliably solved tasks from those
solved less consistently. On WebShop, the association is weaker ($\rho=0.25$--$0.40$), with some zero-success tasks exhibiting high consistency, as analyzed in Appendix~\ref{app:consistency_vs_success}. Overall, the positive associations across all nine pairs support consistency as a behavioral indicator of empirical task reliability.

\begin{table*}[t]
\centering

\begingroup
\setlength{\tabcolsep}{4pt}

\begin{tabular}{@{}lcccccccccc@{}}
\hline
{Methods}
& \multicolumn{3}{c}{{ALFWorld}}
& \multicolumn{3}{c}{{AppWorld}}
& \multicolumn{3}{c}{{WebShop}}
& {Average} \\
\cline{2-4}\cline{5-7}\cline{8-10}
&
{GPT} & {Qwen} & {Gemma}
&
{GPT} & {Qwen} & {Gemma}
&
{GPT} & {Qwen} & {Gemma}
& \\
\hline
Name only
& 0.807 & 0.948 & 0.883
& 0.860 & 0.894 & 0.756
& 0.701 & 0.627 & 0.672
& 0.794 \\
Name + arguments
& 0.907 & 0.962 & 0.922
& 0.857 & 0.883 & 0.754
& 0.718 & 0.636 & 0.664
& 0.811 \\

\hline
\end{tabular}
\endgroup

\caption{
AUROC of our method under two action-equivalence criteria:
matching action types only (name only) or matching both
action types and arguments (name + arguments).
}
\label{tab:name_only}
\end{table*}

\begin{figure*}[t]
    \centering
    \includegraphics[width=\linewidth]
    {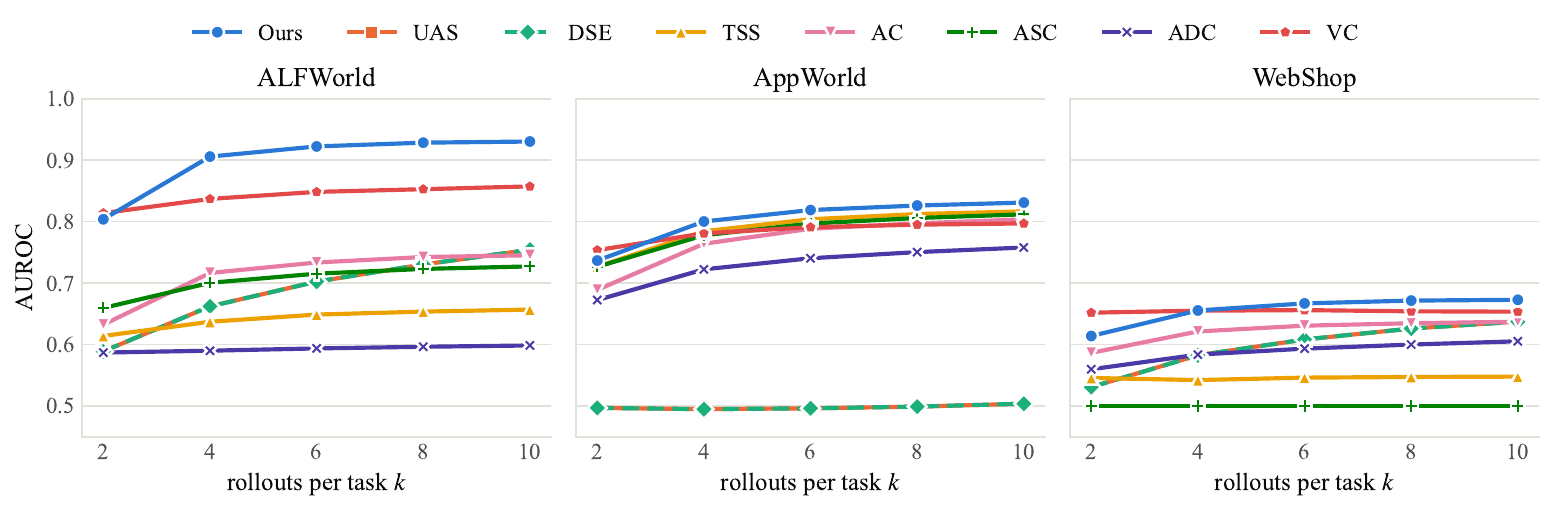}
    \caption{Effect of the number of sampled rollouts $k$ on task-success prediction AUROC across three benchmarks, averaged over three LLMs.}
    \label{fig:rollout_count}
\end{figure*}

\begin{figure*}[t]
    \centering
    \includegraphics[width=1\linewidth]{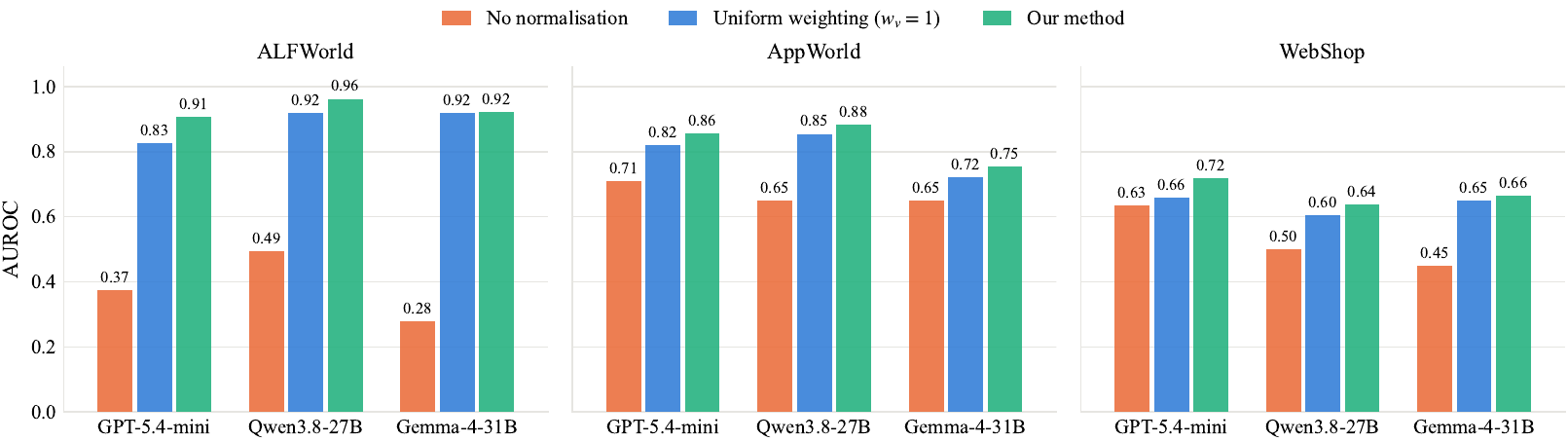}
    \caption{Ablation of episode-length normalization and node weighting. We compare the full method against variants without length normalization or with uniform node weights ($w_v=1$ instead of $w_v=n_v$). }
    \label{fig:ablation_study}
\end{figure*}

\subsection{Consistency predicts success}

We next evaluate whether consistency predicts the task outcome. A task is labeled successful if its success rate over $k=10$ rollouts is at least $0.5$, and a failure otherwise. 
Figure~\ref{fig:consistency_succ_fail} compares the global action consistency scores of the two groups. On ALFWorld and AppWorld the interquartile ranges of succeeding and failing tasks do not overlap on five of the six pairs; on WebShop the medians are ordered but the boxes overlap on all three pairs, indicating weaker separation.

Table~\ref{tab:main_results} compares the AUROC of
our consistency measure with all baselines.
Our method achieves the highest average AUROC,
$0.811$, exceeding verbalized confidence
(VC, $0.769$) and the strongest
consistency baseline (AC, $0.729$).
It ranks first or tied-first in six of the nine
settings and second in the remaining three. Besides, the results show that no consistency-based baseline performs uniformly strong. The sequence-based
measures (TSS, AC, ASC, ADC) perform best overall on
AppWorld but vary widely on ALFWorld and are weaker
on WebShop, where TSS performs near chance and ASC
is at chance. The count-based measures (UAS and
DSE) exhibit the opposite pattern: relatively strong
performance on ALFWorld and WebShop but near-chance
performance on AppWorld. In contrast, our method
achieves the highest or second-highest AUROC in
every setting, maintaining consistently strong
performance across all settings.

Overall scores are lower on WebShop than on the other benchmarks, suggesting that WebShop is the most challenging benchmark for predicting success from consistency, align our findings in Section~\ref{sec:consistency_task_succ_rate} and Figure~\ref{fig:consistency_succ_fail}. One possible reason could be that its episodes are short and follow a fixed skeleton: search, click a product, select options, buy. This means even unsuccessful rollouts may traverse similar steps and exhibit high consistency.

\subsection{Effect of Action-Equivalence Granularity}
The preceding experiments use fine-grained action
equivalence: two actions agree if they call the same
API or command \emph{with the same arguments}.
Here, we consider a coarser \emph{name-only} variant
that ignores arguments. Table~\ref{tab:name_only} reports average AUROCs
of $0.811$ and $0.794$ with and without arguments,
respectively. Retaining arguments improves average
performance on ALFWorld and WebShop, where the command
vocabularies are small and relatively generic.
For example, \texttt{go to} accounts for $41$--$60\%$
of the observed ALFWorld steps, but its destination
argument specifies which location to visit.
Ignoring arguments can therefore conceal decision
differences. On AppWorld, name-only matching performs slightly better.
The policies invoke $173$--$227$ distinct APIs, whose
names already distinguish a broad range of operations. For APIs with multiple arguments, strict matching may
unnecessarily distinguish minor parameter variations.
Thus, the appropriate granularity depends on whether
action names alone are sufficiently informative or
arguments are essential for distinguishing decisions.

\subsection{Effect of the Number of Rollouts}
\label{sec:rollout_count}

We investigate the sensitivity of consistency measures to the
number of sampled rollouts $k\in\{2,4,6,8,10\}$.
For $k<10$, we randomly sample $k$ rollouts without
replacement from the ten recorded rollouts per task,
repeating this 100 times. All methods use the same
sampled rollouts, while task labels remain fixed based
on all ten rollouts.
Figure~\ref{fig:rollout_count} reports the mean task prediction AUROC
across draws, averaged over three LLMs for each benchmark.
Our method generally outperforms the baselines across
different rollout budgets.

\subsection{Ablation Study}

We ablate two design choices in the global consistency
score (Eq.~\ref{equ: global_consistency}): normalisation
by total episode length and visitor-based node weights
$w_v=n_v$. Figure~\ref{fig:ablation_study} reports the
AUROC of each variant across all nine model--benchmark
pairs. Length normalization is important because the unnormalized numerator $\sum_v n_v S_v$ tends to increase with episode length. In ALFWorld, failed tasks have trajectories two to three times longer than successful ones (46--48 vs.\
15--23 steps), causing the unnormalized score to rank failed tasks higher. Node weighting provides a smaller but consistent benefit. Uniform weights assign equal importance to states visited by two or ten rollouts. Removing visitor-based weights decreased the average AUROC by $0.037$, with statistically significant decreases in five of the nine model--benchmark pairs. Together, these results support the effectiveness of both length normalisation and visitor-based weighting in the proposed consistency measure.

\section{Related Work}

\paragraph{Graph Construction for Agentic Reinforcement Learning.}
Recent RL methods leverage shared states and transitions across rollouts to improve step-level credit assignment. GiGPO groups actions from repeated environment states for relative-advantage estimation \citep{feng2026gigpo}. Graph-based methods further organize rollouts into shared structures. GraphGPO constructs a unified state-transition graph and assigns edge credit based on estimated progress toward the goal \citep{cheng2026graphgpo}, while G2PO merges identical observations into shared nodes and estimates edge advantages using aggregated state values \citep{wang2026g2po}. GEPO merges semantically similar observations using sentence embeddings for graph-based credit assignment \citep{yuan2026gepo}. SALT incorporates recent action--observation history into state representations and merges transitions sharing the same source state, action, and successor state \citep{li2026salt}. While these methods use shared state structures for reward-based policy optimization, our work constructs graphs to measure the behavioral consistency of fixed agents across rollouts, without requiring outcome labels at inference time.

\paragraph{Uncertainty Quantification for LLM Agents.}
UQ for LLM agents have built on approaches developed for LLMs, including probability-based scores, verbalized confidence, and sample consistency \citep{kuhn2023semantic,zhang-etal-2025-recognizing,oh2026uncertainty}. Agentic settings introduce additional challenges from function calls and multi-step interactions. UProp separates uncertainty intrinsic to a decision from uncertainty inherited from earlier decisions \citep{duan2025uprop}, while SAUP aggregates step-level uncertainty using situation-aware weights \citep{zhao2025saup}. UOPD utilize model confidence to identify intervention turns \citep{zhang2026uopd}. TRACER combines token surprisal with behavioral signals to identify high-risk episodes \citep{tayebati2026tracer}. \citet{oh2026uncertainty} represent trajectories as a dynamic Bayesian network over actions, states, and observations. RUPA propagates uncertainty along temporal and semantic dependencies in a trajectory graph \citep{ma2026rupa}, while PropUQ-MAS combines local uncertainty with uncertainty inherited through inter-agent communication \citep{liu2026propuq}. Rather than propagating uncertainty through execution dependencies, we quantify uncertainty using action agreement across repeated rollouts.

\paragraph{Behavioral Consistency for LLM Agents.}
Behavioral consistency has gained attention as an important dimension of agent reliability and uncertainty quantification. \citet{rabanser2026reliability} distinguish outcome, trajectory, and consistency within a broader reliability framework. Existing methods measure consistency across repeated executions through action sequence diversity \citep{mehta2026agents}, discrete semantic entropy over sequence clusters~\citep{farquhar2024detecting, ye2026function}, tool-call sequences \citep{yagubyan2026consistent}, or trajectory-level comparisons \citep{bouchard2026beyond}. MATU further examines recurring interaction patterns in multi-agent systems using tensor decomposition \citep{chen2026matu}. While these methods provide valuable insights into agent behavior, existing consistency measures largely rely on sequence-level or aggregate comparisons. In contrast, our work constructs state-action graphs to compare decisions made under equivalent contexts, enabling fine-grained consistency measurements.

\section{Conclusion}

We presented GraphACM, a graph-based framework for measuring behavioral consistency and quantifying uncertainty in LLM agents. By merging repeated executions into a shared state-transition graph, our method enables action comparisons under comparable decision contexts. We introduced node-level and global action consistency measures, supporting both fine-grained behavioral analysis and task-level uncertainty estimation. Experiments across three LLMs and three agent benchmarks demonstrate a positive association between behavioral consistency and task success. Our method also outperforms existing behavioral-consistency baselines and verbalized confidence in average task-success prediction performance. These findings highlight the effectiveness of state-conditioned behavioral consistency as an interpretable, training-free signal for assessing agent reliability and quantifying uncertainty.

\section*{Limitations}
Our framework has several limitations. First, behavioral consistency serves as an indirect indicator of reliability rather than a guarantee of correctness. Agents may consistently repeat incorrect decisions, as observed in some WebShop tasks. This is an inherent limitation shared by consistency-based approaches rather than a limitation specific to our method. Second, the effectiveness of state-aligned comparisons depends on the quality of state abstraction. Incorrectly merging distinct decision contexts or failing to identify equivalent states may affect consistency estimates. Third, our method requires multiple rollouts of the same task, introducing additional inference costs compared with single-execution approaches. Finally, our evaluation is limited to three LLMs and three agent benchmarks. Future work could investigate the sensitivity of our method to the state representations, as well as its generalization to more diverse agent environments.

\section*{Ethical Considerations}
High behavioral consistency does not guarantee correct
or safe behavior and should not be used
alone to authorize consequential agent actions.
Repeated rollouts also increase computational demands
and may increase environmental costs.



\bibliography{custom}

\newpage
\appendix

\section{Consistency Measurement}

\subsection{Boundedness of Task-Level Consistency}
\label{app:sg_bound}

\begin{proposition}
For $k$ rollouts with $\sum_{i=1}^{k}T_i>0$,
the task-level consistency score satisfies
$0\leq S_G\leq 1$.
\end{proposition}

\begin{proof}
By definition, $0\leq S_v\leq 1$ and $w_v=n_v\geq 0$.
Since $n_v$ counts the number of distinct rollouts
visiting node $v$, each rollout contributes at most
$T_i$ to the sum of node weights. Therefore,
\begin{equation*}
    \sum_{v\in\mathcal V}n_v
    \leq \sum_{i=1}^{k}T_i
    =k\bar T.
\end{equation*}
It follows that
\begin{equation*}
    0\leq S_G
    =\frac{\sum_{v\in\mathcal V}n_vS_v}{k\bar T}
    \leq\frac{\sum_{v\in\mathcal V}n_v}{k\bar T}
    \leq 1.
\end{equation*}
Thus, $S_G\in[0,1]$.
\end{proof}

\subsection{Proof of Local Consistency and Its Entropy Interpretation}
\label{app:local-consistency-proof}

Consider a shared state $v$ with $n_v\geq2$ supporting rollouts.
Let $A_1,\ldots,A_{n_v}$ denote their representative actions,
assumed to be i.i.d.\ from the conditional action distribution
$p_v$, where $p_v(a)=\Pr(A_i=a)$.
All expectations and probabilities are conditioned on fixed
$v$ and $n_v$.

Local consistency measures the average pairwise action agreement:
\begin{equation*}
    S_v = \frac{1}{\binom{n_v}{2}}
    \sum_{i<j}\mathbf{1}[A_i=A_j].
\end{equation*}

\begin{proof}[Proof of Lemma~\ref{lem:local-consistency}]
By linearity of expectation and independence,
\begin{equation*}
\begin{aligned}
    \mathbb{E}[S_v]
    &= \frac{1}{\binom{n_v}{2}}
       \sum_{i<j}\mathbb{E}[\mathbf{1}[A_i=A_j]] \\
    &= \frac{1}{\binom{n_v}{2}}
       \sum_{i<j}\Pr(A_i=A_j) \\
    &= \frac{1}{\binom{n_v}{2}}
       \sum_{i<j}\sum_a \Pr(A_i=a,A_j=a) \\
    &= \frac{1}{\binom{n_v}{2}}
       \sum_{i<j}\sum_a p_v(a)^2 \\
    &= \sum_a p_v(a)^2.
\end{aligned}
\label{eq:local-consistency-proof}
\end{equation*}
The final expression equals the probability that two
independent actions drawn from $p_v$ agree.
Thus, $S_v$ is an unbiased estimator of the pairwise
action-agreement probability.

Furthermore, the Rényi entropy of order two is defined as
\begin{equation*}
    H_2(p_v)=-\log\sum_a p_v(a)^2.
\end{equation*}
It follows that
\begin{equation*}
    \mathbb{E}[S_v]
    =\sum_a p_v(a)^2
    =\exp(-H_2(p_v)).
\end{equation*}
Therefore, expected local consistency is monotonically
decreasing in the Rényi entropy of the conditional
action distribution.
\end{proof}

\section{Graph Construction}
\label{app:graph}

Table~\ref{tab:statekey} summarizes the components of the
abstract state $z_t=\phi(e_t,h_t)$ for each benchmark.
For ALFWorld, we adopt the state representation of
GraphGPO~\citep{cheng2026graphgpo}, where $e_t$ includes
the observation, agent location, held object and its status,
and admissible commands, while $h_t$ records the history
of moved objects using GraphGPO's rule-based tracker.
For AppWorld, we classify API calls as state-changing or
read-only based on the benchmark's OpenAPI specifications.
The environment component $e_t$ consists of the sequence
of successful state-changing calls, representing database
updates, while $h_t$ records answered read calls as a proxy
for previously acquired information. For WebShop, $e_t$ captures the current page, product,
view, and selected options, while $h_t$ records previously
opened product pages and sections, reflecting the agent's
exploration history.
Together, these abstractions enable actions to be compared
under equivalent abstract decision contexts across rollouts.

\begin{table}[t]
\centering\small
\begin{tabular}{@{}p{1.3cm}p{2.65cm}p{2.45cm}@{}}
\toprule
Benchmark & Environment context $e_t$ & Accumulated knowledge $h_t$ \\
\midrule
ALFWorld & Observation, location, held object, admissible commands & Moved-object history \\
AppWorld & Successful state-changing API calls & Answered read calls \\
WebShop  & Page, product, view, selected options & Opened product pages and sections \\
\bottomrule
\end{tabular}
\caption{Benchmark-specific components of the abstract state $z_t=\phi(e_t,h_t)$.}
\label{tab:statekey}
\end{table}



\begin{figure*}[t]
    \centering
    \includegraphics[width=0.8\linewidth]{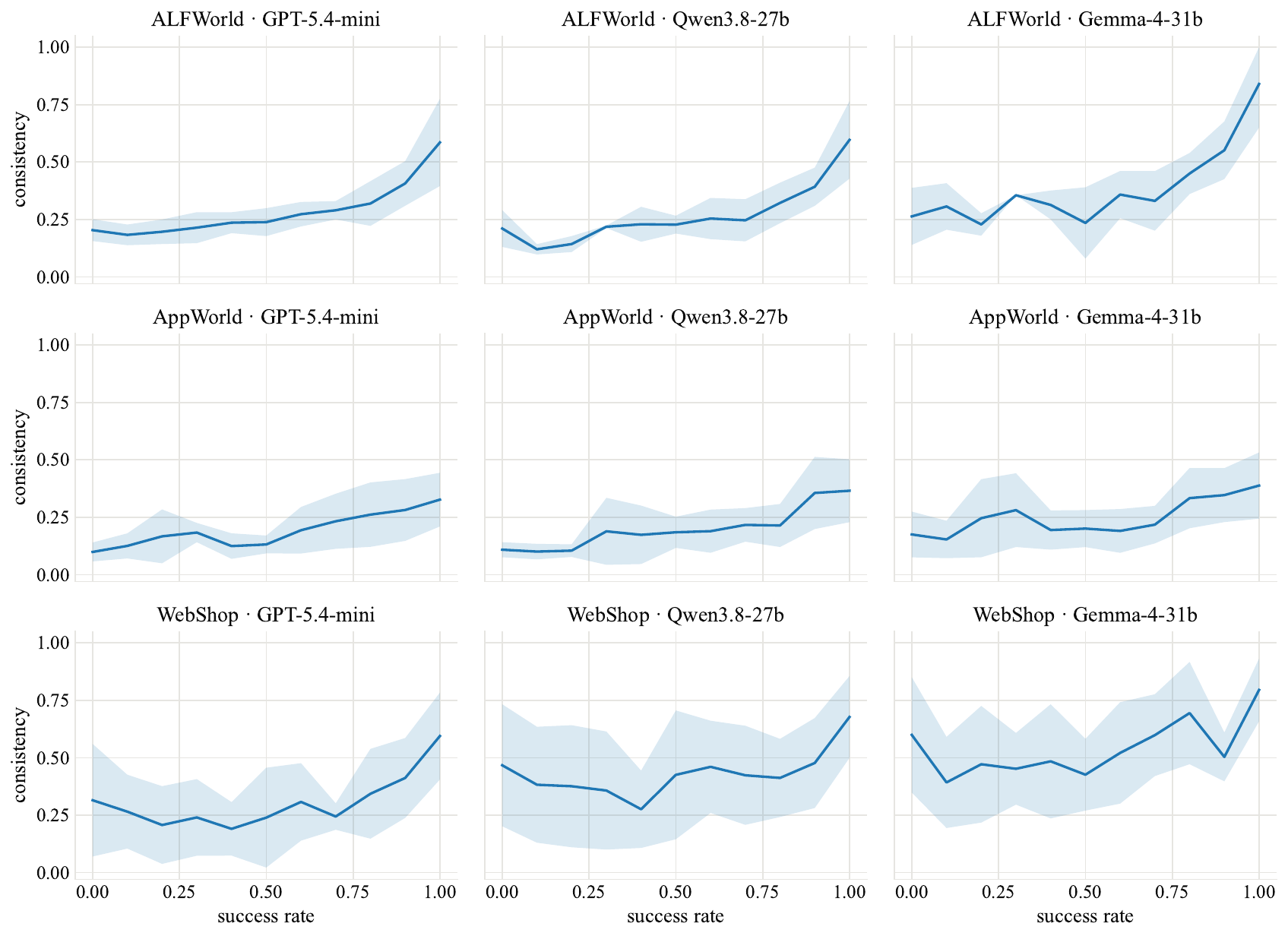}
    \caption{Global action consistency versus task success rate across three benchmarks and three models. Success rate is computed from $k=10$ rollouts per task--model pair.}
    \label{app_fig:consistency_success_rate}
\end{figure*}

\section{Experiment Settings}
\label{app:exp-details}
\paragraph{Benchmark and Agent Configuration.}
We use all 134 ALFWorld tasks from the
\texttt{valid\_unseen} split, covering six household
task categories; 168 AppWorld tasks from the
\texttt{test\_normal} split; and the first 200 WebShop
test instructions. All tasks are in English.
Generation is capped at 512 tokens per response for
ALFWorld and 1,024 for AppWorld and WebShop. Rollout success follows ALFWorld's goal check,
requires all AppWorld unit tests to pass, and
requires a terminal reward of $1$ in WebShop.

\paragraph{Computing Infrastructure and Artifact Use.}
Qwen runs locally on four NVIDIA RTX A6000 GPUs requiring
approximately 125 GPU-hours for final runs,
while Gemma and GPT-5.4-mini are accessed via APIs with total budgets of about \$45 and \$245, respectively. Parameter counts of the open-weight models are given by their names (27B, 31B); that of GPT-5.4-mini is not public. We use ALFWorld, WebShop,
and AppWorld for their intended research purposes.

\section{Experimental Results}
\subsection{Basic Trajectory Statistics}
\label{app:basic}

Table~\ref{tab:bas_res} reports success rates and mean
trajectory lengths for three models across ALFWorld,
AppWorld, and WebShop.
On ALFWorld, Qwen3.8-27B achieves the highest success rate
(80.52\%) and the shortest mean trajectory length (21.66 steps).
On AppWorld, GPT-5.4-mini achieves the highest success rate
(65.30\%) with the fewest steps on average (15.90).
Across all three models, WebShop exhibits lower success
rates (40.10\%--43.60\%) and shorter mean trajectories
(10.52--13.08 steps) than the other benchmarks.

\begin{table}[t]
\centering
\footnotesize
\setlength{\tabcolsep}{3.5pt}
\renewcommand{\arraystretch}{1.1}
\label{tab:trajectory_statistics}
\begin{tabular}{@{}llrr@{}}
\toprule
\textbf{Dataset} & \textbf{Model}
& \textbf{Success (\%)}
& \textbf{Mean steps} \\
\midrule

\multirow{3}{*}{ALFWorld} & Qwen3.8-27B    & 80.52 & 21.66 \\
                          & Gemma-4-31B-IT & 67.01 & 25.27 \\
                          & GPT-5.4-mini   & 53.88 & 33.23 \\

\midrule
\multirow{3}{*}{AppWorld} & Qwen3.8-27B    & 63.33 & 28.36 \\
                          & Gemma-4-31B-IT & 56.43 & 22.82 \\
                          & GPT-5.4-mini   & 65.30 & 15.90 \\
\midrule
\multirow{3}{*}{WebShop}
& Qwen3.8-27B     & 42.35 & 10.52 \\
& Gemma-4-31B  & 40.10 & 11.54 \\
& GPT-5.4-mini    & 43.60 & 13.08 \\
\bottomrule

\end{tabular}
\caption{Basic trajectory statistics across ALFWorld,
AppWorld, and WebShop.}
\label{tab:bas_res}
\end{table}

\begin{table*}[t]
\centering

\begingroup
\small                          
\setlength{\tabcolsep}{5pt}
\renewcommand{\arraystretch}{0.95}
\newcommand*{\tabci}[2]{{\scriptsize[#1, #2]}}   

\begin{tabular}{@{}lcccccccccc@{}}
\hline
{Methods}
& \multicolumn{3}{c}{{ALFWorld}}
& \multicolumn{3}{c}{{AppWorld}}
& \multicolumn{3}{c}{{WebShop}}
& {Average} \\

\cline{2-4}\cline{5-7}\cline{8-10}

&
{GPT} & {Qwen} & {Gemma}
&
{GPT} & {Qwen} & {Gemma}
&
{GPT} & {Qwen} & {Gemma}
& \\

\hline

UAS
& 0.655 & 0.741 & 0.866
& 0.476 & 0.504 & 0.531
& 0.702 & 0.592 & 0.618
& 0.632 \\ \noalign{\vskip -2pt}
& \tabci{.60}{.71} & \tabci{.70}{.79} & \tabci{.82}{.91}
& \tabci{.43}{.51} & \tabci{.50}{.51} & \tabci{.50}{.56}
& \tabci{.63}{.77} & \tabci{.51}{.67} & \tabci{.54}{.70}
& \tabci{.61}{.66} \\[2pt]

DSE
& 0.655 & 0.741 & 0.866
& 0.476 & 0.504 & 0.531
& \underline{0.703} & 0.596 & 0.613
& 0.632 \\ \noalign{\vskip -2pt}
& \tabci{.60}{.71} & \tabci{.70}{.79} & \tabci{.82}{.91}
& \tabci{.43}{.51} & \tabci{.50}{.51} & \tabci{.50}{.56}
& \tabci{.64}{.77} & \tabci{.52}{.67} & \tabci{.54}{.69}
& \tabci{.61}{.66} \\[2pt]

TSS
& 0.523 & 0.777 & 0.670
& 0.847 & 0.849 & \textbf{0.754}
& 0.519 & 0.553 & 0.571
& 0.674 \\ \noalign{\vskip -2pt}
& \tabci{.42}{.63} & \tabci{.67}{.87} & \tabci{.58}{.76}
& \tabci{.78}{.91} & \tabci{.77}{.92} & \tabci{.68}{.83}
& \tabci{.44}{.60} & \tabci{.47}{.64} & \tabci{.49}{.65}
& \tabci{.64}{.71} \\[2pt]

AC
& 0.730 & 0.728 & 0.777
& 0.819 & \textbf{0.888} & 0.704
& 0.683 & 0.611 & 0.618
& 0.729 \\ \noalign{\vskip -2pt}
& \tabci{.65}{.81} & \tabci{.62}{.83} & \tabci{.69}{.86}
& \tabci{.75}{.89} & \tabci{.82}{.95} & \tabci{.62}{.78}
& \tabci{.61}{.75} & \tabci{.53}{.69} & \tabci{.54}{.69}
& \tabci{.70}{.76} \\[2pt]

ASC
& 0.509 & \underline{0.900} & 0.773
& \textbf{0.878} & 0.835 & 0.722
& 0.500 & 0.500 & 0.500
& 0.680 \\ \noalign{\vskip -2pt}
& \tabci{.40}{.61} & \tabci{.83}{.95} & \tabci{.67}{.86}
& \tabci{.82}{.93} & \tabci{.76}{.90} & \tabci{.64}{.80}
& \tabci{.50}{.50} & \tabci{.50}{.50} & \tabci{.50}{.50}
& \tabci{.66}{.70} \\[2pt]

ADC
& 0.439 & 0.677 & 0.679
& 0.841 & 0.746 & 0.687
& 0.630 & 0.594 & 0.592
& 0.654 \\ \noalign{\vskip -2pt}
& \tabci{.34}{.54} & \tabci{.52}{.81} & \tabci{.58}{.77}
& \tabci{.77}{.90} & \tabci{.66}{.83} & \tabci{.60}{.77}
& \tabci{.55}{.70} & \tabci{.51}{.67} & \tabci{.51}{.67}
& \tabci{.62}{.69} \\[2pt]

VC
& \underline{0.829} & 0.853 & \underline{0.890}
& 0.791 & 0.873 & \underline{0.727}
& 0.633 & \textbf{0.667} & \underline{0.660}
& \underline{0.769} \\ \noalign{\vskip -2pt}
& \tabci{.76}{.89} & \tabci{.78}{.92} & \tabci{.83}{.94}
& \tabci{.70}{.87} & \tabci{.80}{.94} & \tabci{.65}{.81}
& \tabci{.56}{.71} & \tabci{.59}{.74} & \tabci{.60}{.72}
& \tabci{.74}{.80} \\[2pt]


Ours
& \textbf{0.907} & \textbf{0.962} & \textbf{0.922}
& \underline{0.857} & \underline{0.883} & \textbf{0.754}
& \textbf{0.718} & \underline{0.636} & \textbf{0.664}
& \textbf{0.811} \\ \noalign{\vskip -2pt}
& \tabci{.85}{.95} & \tabci{.93}{.99} & \tabci{.87}{.96}
& \tabci{.79}{.91} & \tabci{.82}{.94} & \tabci{.68}{.83}
& \tabci{.65}{.78} & \tabci{.56}{.71} & \tabci{.59}{.74}
& \tabci{.79}{.84} \\

\hline
\end{tabular}

\endgroup

\caption{AUROC for predicting task success across nine model--benchmark pairs. Higher is better. Brackets show 95\% confidence intervals from 2,000 task-level bootstrap resamples, shared by all methods within each column. For the average AUROC, each benchmark's task resample is shared across its three models. Best results are shown in bold and second-best results are underlined.}
\label{app_tab:main_results_ci}
\end{table*}

\subsection{Consistency and Task Success Rate}
\label{app:consistency_vs_success}

Figure~\ref{app_fig:consistency_success_rate} extends the analysis in Section~\ref{sec:consistency_task_succ_rate} to all nine model--benchmark pairs. On ALFWorld and AppWorld, consistency generally increases with success rate across all three models. These patterns suggest that tasks solved successfully in more rollouts also tend to elicit more consistent decisions, supporting the positive associations reported in the main text. WebShop exhibits a similar but weaker positive trend. Notably, some zero-success tasks exhibit high consistency, reflecting consistent but unsuccessful behavior.


\subsection{Consistency Predicts Success}
\label{app:consistency_predicts_success}

Table~\ref{app_tab:main_results_ci} supplements Table~\ref{tab:main_results} with 95\% confidence intervals from 2,000 task-level bootstrap resamples. Resamples are shared across methods within each setting and across models within each benchmark when computing the average AUROC, preserving task-level dependencies. Our method achieves the highest average AUROC ($0.811$, $[0.79,0.84]$), with its lower bound exceeding the upper bound of every consistency-based baseline, including AC ($0.729$, $[0.70,0.76]$). These results support our method’s effectiveness in predicting task success and its advantage over existing consistency-based baselines.

\end{document}